\documentclass{article} 
\usepackage{iclr2027_conference,times}
\iclrfinalcopy

\usepackage[hyphens]{url}
\usepackage{graphicx}
\usepackage{caption}
\usepackage{algorithm}
\usepackage{algorithmic}

\usepackage{amsfonts}
\usepackage{amsmath}
\usepackage{amssymb}
\usepackage{amsthm}
\usepackage{mathtools}
\usepackage{dsfont}
\usepackage{booktabs}

\theoremstyle{plain}
\newtheorem{theorem}{Theorem}[section]

\theoremstyle{definition}
\newtheorem{definition}[theorem]{Definition}

\usepackage{amsmath,amsfonts,bm}

\def\eqref#1{equation~\ref{#1}}

\def\1{\bm{1}}

\DeclareMathAlphabet{\mathsfit}{\encodingdefault}{\sfdefault}{m}{sl}
\SetMathAlphabet{\mathsfit}{bold}{\encodingdefault}{\sfdefault}{bx}{n}

\usepackage{hyperref}
\usepackage{url}

\title{SHERLOC: Statistical Hypothesis Evaluation and Recovery for Learning Opponent Characteristics
}

\author{Soyon Choi \\
Department of Computer Science\\
Vanderbilt University\\
Nashville, TN, USA \\
\texttt{\{soyon.choi\}@vanderbilt.edu} \\
\And
Scott Alfeld \\
Department of Computer Science \\
Amherst College \\
Amherst, MA, USA \\
\texttt{\{salfeld\}@amherst.edu} \\
\And
Meiyi Ma \\
Department of Computer Science\\
Vanderbilt University \\
Nashville, TN, USA \\
\texttt{\{meiyi.ma\}@vanderbilt.edu} \\
}

\begin{document}

\maketitle

\newcommand{\attacker}{\texttt{ATKR}}
\newcommand{\defender}{\texttt{DFDR}}

\newcommand{\knowledge}{\mathcal{K}}
\newcommand{\capability}{\mathcal{C}}
\newcommand{\objective}{\mathcal{O}}

\newcommand{\bx}{\pmb{x}}
\newcommand{\bX}{\pmb{X}}
\newcommand{\bY}{\pmb{Y}}
\newcommand{\bz}{\pmb{z}}
\newcommand{\bp}{\pmb{p}}

\newcommand{\model}{\pmb{M}}
\newcommand{\bW}{\pmb{W}}
\newcommand{\bC}{\pmb{C}}
\newcommand{\bV}{\pmb{V}}
\newcommand{\bG}{\pmb{G}}
\newcommand{\balpha}{\pmb{\alpha}}

\newcommand{\bP}{\pmb{P}}
\newcommand{\bq}{\pmb{q}}
\newcommand{\br}{\pmb{r}}

\newcommand{\bA}{\pmb{A}}
\newcommand{\bv}{\pmb{v}}
\newcommand{\bs}{\pmb{s}}

\newcommand{\bD}{\pmb{D}}
\newcommand{\datap}{\mathcal{P}_{data}}

\newcommand{\mean}{\mu}
\newcommand{\bmean}{\pmb{\mu}}
\newcommand{\bSig}{\pmb{\Sigma}}
\newcommand{\bPsi}{\pmb{\Psi}}
\newcommand{\std}{\sigma}

\newcommand{\tr}{\text{Tr}}
\newcommand{\bw}{\pmb{w}}
\newcommand{\bc}{\pmb{c}}

\newcommand{\identity}{\pmb{\mathcal{I}}}

\newif\ifincludedefense
\includedefensetrue

\begin{abstract}

When used in automated decision-making systems, machine learning (ML) models are vulnerable to data-manipulation attacks. 
Some defense mechanisms (e.g., adversarial regularization) directly affect the ML models while others (e.g., anomaly detection) act within the broader system.
However, many defenses operate under a predefined threat model that assumes characteristics of the adversary, including their knowledge, capabilities, and objectives, are known and fixed. 
In practice, these characteristics may be unknown or evolve over time, creating a mismatch between the assumptions of the defender and the true attacker.

In this paper, we study a new perspective on adversarial defense: 
focusing on the attack\textit{er}, rather than the attack. 
We present and demonstrate \textbf{SHERLOC}, a framework for adversarial learning forensics -- namely, a \textit{detective} that learns about the attacker from observed attack(s).

We prove that, without additional knowledge, the attacker is non-identifiable --- multiple potential attackers would perform the same observed attack. 
This reveals an important distinction between \textit{reconstructing} an attacker's behavior and \textit{recovering} the attacker that generated it.
We demonstrate empirically that SHERLOC is able to reconstruct a hypothesis attacker that is consistent with the observed attacks against various defender models (i.e., generalized least-squares linear regressors, logistic regressors, and polynomial classifiers) across various security-critical domains.
Finally, we show that these inferred attackers need not match the ground truth attacker exactly to have downstream utility.
Defenses using the SHERLOC-inferred attackers achieve approximately $90-100\%$ of robustness improvement that would be achieved if the defender had access to the ground-truth attacker.

\end{abstract}

\section{Introduction}

Despite their capabilities, machine learning (ML) systems are vulnerable to adversarial attacks --- small perturbations to input data that cause the model to predict incorrectly \citep{szegedy2014intriguingpropertiesneuralnetworks, DBLP:journals/corr/Moosavi-Dezfooli15, dabouei2019smoothfoolefficientframeworkcomputing}.
However, many defenses are designed around a fixed threat model that makes assumptions about the adversary, including their knowledge of the target system, their capabilities and limitations, and their objective \citep{pmlr-v80-athalye18a, pmlr-v97-cohen19c}.
When these assumptions do not match the adversary, the resulting defense provides limited protection, contributing to an ongoing arms race between increasingly sophisticated attacks and defenses \citep{NEURIPS2020_11f38f8e}.

This reliance on a predefined attacker presents a fundamental challenge in real-world settings, where an adversary's knowledge, capabilities, and objectives are non-stationary and/or unknown.
Robustness against such adversaries therefore requires moving beyond defenses that assume these characteristics about the attacker. 
Instead, we ask whether information about an attack\textit{er} can be inferred from the attacks themselves, enabling the defender to reason about an otherwise unknown or evolving threat model.

\ifincludedefense
To this end, this paper pursues two distinct but complementary goals.

First, we seek to understand what can be learned about an attacker from the attacks they leave behind.
The importance of “knowing one’s enemy” has long been recognized in adversarial settings \citep{tzu512art}, and in real-world ML systems, knowledge about the attacker can inform risk assessment, attack detection, and robust design of future systems.
In this paper, we focus on attackers against simple learners such as linear and polynomial models, which are often favored in safety-critical applications for their interpretability, transparency, and ease of validation. 
Under these settings, we (i) highlight the limitations of exact attacker recovery and (ii) investigate the ability to reconstruct hypothesis attackers that behave consistently with the observed attacks via optimization (i.e., SHERLOC-inferred attackers).
Thus, our first goal is to characterize what information about an attacker can be \textbf{recovered} vs. \textbf{reconstructed} from their observed behavior, and under what conditions.

Second, we ask whether the information inferred about an attacker can be used to construct a more informed defense.
In practice, the attacker's beliefs, objective, and capabilities are almost never known by the defender; the only information that is reliably known by the defender is the actual input to the defender's model.
Hence, rather than assuming fixed attacker parameters, the defender should infer them from the attacks they observe and adapt their defensive strategy accordingly.
As one concrete instantiation, we use the attacker parameters inferred by our framework to construct an attacker-specific adversarial regularizer.

\else
Unlike traditional defense goals, we instead aim to better understand the attacker itself. 
In real-life ML applications, learning about the knowledge, objectives, and capabilities of the attacker provides insight into the nature of the overall threat.
The importance of “knowing one’s enemy” has long been recognized in adversarial settings \citep{tzu512art}, and in the ML context, such knowledge can inform risk assessment, attack detection, and robust design of future systems.
Hence, as in all security settings, post-attack forensics is valuable even independent of any particular defense mechanism.
\fi

Concretely, we make the following contributions:
\begin{enumerate}
    \item We establish a fundamental limitation of attacker forensics: distinct attackers can induce the same attack behavior and are therefore non-identifiable from attack observations alone. For linear learners, we mathematically illustrate this phenomenon, yielding non-trivial equivalence classes of behaviorally indistinguishable attackers.
    \item We introduce SHERLOC, a general framework and method for inferring attacker hypotheses behaviorally consistent with the observed attack(s). Motivated by non-identifiability, 
    SHERLOC distinguishes between reconstructing attacker behavior and 
    recovering the ground-truth attacker parameters.
    \ifincludedefense
    \item We empirically demonstrate SHERLOC on linear and polynomial learning methods, then apply the inferred attacker to defense.
    We show that SHERLOC can accurately reconstruct attacker behavior even when the ground-truth attacker is not exactly recovered, and that the inferred
attackers remain useful for downstream defense, achieving approximately 90–100\% of the
robustness improvement obtained with access to the ground-truth attacker.
    \else 
    \item We empirically demonstrate our approach on linear learning methods across various datasets from security-critical domains.
    \fi
\end{enumerate}

In total these contributions help to solidify the foundation of the RED (Reverse Engineering of Deception) research area. 



\section{Threat Model}

We consider two players, the defender/learner \defender{} and the attacker/adversary \attacker{}.
At deployment time, \attacker{} manipulates inputs to \defender{}, while \defender{} observes the resulting attack(s) and aims to infer properties of the attacker that generated them.
Based on the classical taxonomy widely used in adversarial machine learning \citep{10.1145/1128817.1128824, 10.1007/s10994-010-5188-5, 10.1145/2046684.2046692}, an attacker is defined by three fundamental aspects: what it \textit{knows} about the target model, what actions it \textit{can} perform, and what it \textit{aims} to achieve.
In our framework, we model these aspects of \attacker{} as the tuple
\[
(\knowledge, \capability, \objective)
\]
where:
\begin{enumerate}
    \item \(\knowledge\): Their knowledge about \defender{}.
    What does \attacker{} think \defender{}'s prediction function is?
  
    \item \(\capability\): Their capability. 
    What perturbations to the data can \attacker{} make?

    \item \(\objective\): Their objective.
    What function does \attacker{} try to optimize?
\end{enumerate}

This decomposition is central to SHERLOC. 
Beyond corresponding to the standard conceptual dimensions of an attacker, $\knowledge, \capability, \objective$ play distinct roles in the process by which an attack is generated: 
$\knowledge$ determines the model whose behavior the attacker reasons about, $\capability$ determines which manipulations are feasible, and $\objective$ determines which outcomes the attacker prefers. 
Consequently, inferring each component constitutes a qualitatively different inverse problem.

In the forward problem, \attacker{} uses $(\knowledge, \capability, \objective)$ to generate an attack.
\attacker{} aims to perturb the input within its capability constraint $\capability$ such that the output of its knowledge model $\knowledge$ shifts as far as possible in the directions weighted by its objective $\objective$.


We denote the optimal attack on a benign input $\bx$ as the following function of the input and \attacker{}'s parameters:
\begin{align}
    \balpha_{opt}(\bx ; \attacker{}) &= \balpha_{opt}(\bx ; \knowledge, \capability, \objective) \\
    &= \arg\max_{\balpha \in \mathcal{A}_{\capability}} g_{ \objective} (f_{\knowledge}, \bx, \balpha)
\end{align}
\noindent where: 
\begin{itemize}
    \item $f_{\knowledge}$ is the prediction function of the model used by \attacker{},
    \item $\mathcal{A}_{\capability}$ is the set of attacks \attacker{} is capable of,
    \item and $g_{\objective}$ is the gain function defined by \attacker{}'s objective.
\end{itemize}

In the inverse problem, which is the focus of SHERLOC, \defender{} observes a series of resulting attack(s):
\begin{equation}
    \bA_{obs}(\bX_k; \attacker{}) = \{\balpha_{obs}(\bx; \attacker{}) ~:~ \bx \in \bX_k\}
\end{equation}
\noindent where $\bX_k = (\bx_1,\dots,\bx_k)$ is a set of distinct benign inputs and $\balpha_{obs}(\cdot)$ may or may not be optimal.

Then, it reasons backward about the $(\knowledge, \capability, \objective)$ that could have produced it.
Importantly, solving the inverse problem need not uniquely identify the true attacker $\attacker{}^*$. 
In Section 3, we show that an attacker is non-identifiable --- distinct parameterizations may induce identical attacks.
We therefore distinguish throughout the paper between {\it reconstructing} an attacker's behavior (i.e., finding parameters that explain the observed attacks) and {\it recovering} parameters that are close to the attacker's true $(\knowledge^*,\capability^*,  \objective^*)$. 

We begin by instantiating this framework for linear learners, including linear regression and logistic regression. 
The linear setting is useful for analysis, but exposes a fundamental limitation of reverse-engineering exact attacker parameters from observed attacks --- linearity severely restricts the information exposed by an observed attack.
We then extend our study to polynomial learners, where observing multiple attacks against different inputs can provide repeated evidence about the attacker.
The specific instantiations and corresponding \attacker{} parameterizations in these three settings are presented in Appendix \ref{app:params}.
Finally, we ask whether exact parameter recovery is necessary for a downstream defender: even when \defender{} cannot uniquely identify \attacker{}, the recovered representation may capture enough of the attacker's behavior to support effective defense.

In Section 3, we discuss the key challenges of reverse-engineering \attacker{} and characterize the resulting inference problem.
In Section 4, we demonstrate SHERLOC's empirical behavior and applicability to defense across linear and polynomial learners.
Together, these settings illustrate both the theoretical foundation and the practical applicability of our approach.

Note that we assume \defender{} knows the parameterized form of the attacker, but not the specific parameter values, which are inferred from the observed attack using our framework. 
The flexibility of these parameters allows our framework to be domain-agnostic, freely admitting other attackers, defenders, and parameterizations thereof.


\begin{figure*}[t!]
    \centering
    \includegraphics[width=0.7\linewidth]{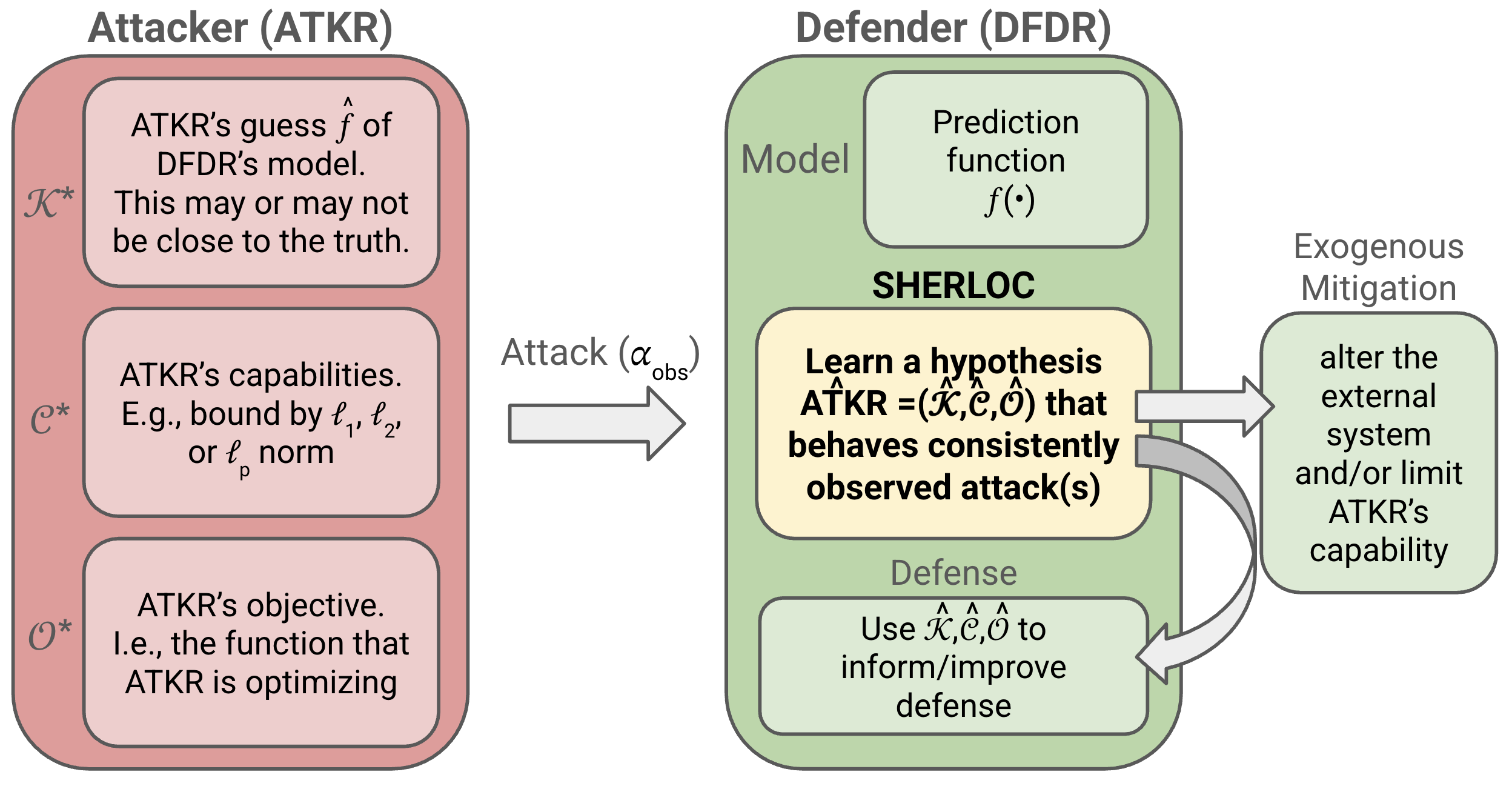}
    \caption{Schematic overview of our framework within the overall attacker-defender system. 
    In this paper, we consider the example cases in which $f$ is a linear regression, logistic regression, or polynomial regression model.
    }
    \label{fig:framework_overview}
\end{figure*}

\section{Recovering Attacker Parameters}

Suppose \defender{} has no prior knowledge (other than the parameterization) of \attacker's knowledge ($\knowledge$), capability\footnote{Except that the observed attack is clearly something \attacker{} is capable of.} ($\capability$), or objective ($\objective$). 
We now consider the inverse problem faced by \defender{}: Given one or more observed attacks, what can \defender{} infer about the \((\knowledge, \capability, \objective)\) that generated them?

At first glance, one might formulate this problem by searching for \((\knowledge, \capability, \objective)\) for which the observed attacks are optimal. 
However, reproducing an observed attack does not necessarily identify the exact attacker that generated it.
Multiple values of \(\knowledge, \capability, \objective\) may induce exactly the same optimal attack. 
This distinction between \emph{explaining an attack} and \emph{identifying its source \attacker{}} is fundamental to the reverse-engineering problem.

In this section, we aim to answer the question: \textbf{What can be learned about an attacker from the attacks they leave behind?}
We first formalize this problem and show that, even in the controlled linear setting, exact attacker parameters are generally non-identifiable from an observed attack.
We then introduce \textbf{SHERLOC}, which treats reverse engineering as inference over the equivalence class of attacker hypotheses consistent with the observed behavior. 
This perspective also motivates two distinct notions of success that we evaluate empirically: whether the inferred attacker hypothesis $\hat\attacker{}$ can \emph{reconstruct} the observed attack behavior, and whether its inferred parameters actually \emph{recover} those of the ground-truth attacker.

\subsection{Non-identifiability Property of Attacker}
The most canonical question is: is the attacker identifiable? 
That is, does the observed attack uniquely determine the parameters of \attacker?
Intuitively (and as we show mathematically below) the answer is, in general, no --- two different attackers (e.g., suppose both their objectives and capabilities differ) may end up performing the same attack.
While one can easily construct particular examples to show instances of this phenomenon, in what follows we show a much stronger result for linear attackers.
That is, for {\it any} attack $\balpha$ and any fixed two-of-the-three parameters (\(\knowledge, \capability, \objective\)), there is a set of attackers that are behaviorally indistinguishable without additional information.
%

We define the equivalence class $\mathcal{S}_{\balpha}$ of attackers for an attack vector \(\balpha\) as follows:
\begin{definition}[Attacker equivalence class]
    Let \attacker{} be defined by parameters \((\knowledge, \capability, \objective)\). 
    The optimal attack for \attacker{} is denoted $\balpha_{opt}(\knowledge, \capability, \objective)$. We define the set \(  \mathcal{S}_{\balpha}\) as:
    \begin{equation}
        \attacker \in \mathcal{S}_{\balpha} ~ \textbf{iff} ~ \balpha = \balpha_{opt}(\knowledge, \capability, \objective)
    \end{equation}
\end{definition}

\begin{theorem}[Non-identifiablity of attacker]
    \label{thm:nonidentifiability}
    For a linear \defender{}, the following hold for any attack $\balpha$:
    \begin{enumerate}
        \item $\forall \knowledge, \capability,~ \exists \objective ~ \text{s.t.} ~ \attacker \in \mathcal{S}_{\balpha}, ~ ~$
        \item $\forall \knowledge, \objective,~ \exists \capability ~ \text{s.t.} ~ \attacker \in \mathcal{S}_{\balpha}, ~ ~$
        \item $\forall \capability, \objective,~ \exists \knowledge ~ \text{s.t.} ~ \attacker \in \mathcal{S}_{\balpha}$
    \end{enumerate}
\end{theorem}


The full proof of Theorem \ref{thm:nonidentifiability} is presented in Appendix \ref{app:thm3.2_proof}.

Theorem \ref{thm:nonidentifiability} states that the tuple of \attacker{} parameters \((\knowledge, \capability, \objective)\) is \textbf{non-identifiable}.
It is a fundamental limitation on what can be inferred from attack observations alone. 
Even if \defender{} knows two components of the attacker exactly, observing a single optimal attack does not, in general, uniquely determine the third. 
For any fixed values of two components, the remaining component can be chosen so that the same attack remains optimal. 
Thus, an observed attack identifies an \emph{equivalence class} of possible attackers rather than a unique point in the space of attacker parameters.

This observation has an important consequence for how reverse engineering should be evaluated.
Suppose an inference procedure returns parameters \((\hat\knowledge, \hat\capability, \hat\objective)\) that reproduce the observed attack.
Theorem \ref{thm:nonidentifiability} says that there is no guarantee that \((\hat\knowledge, \hat\capability, \hat\objective) = (\knowledge^*, \capability^*, \objective^*)\), where \((\knowledge^*, \capability^*, \objective^*)\) is the tuple of true $\attacker^*$ parameters.

We address this challenge with SHERLOC, a novel framework in which \defender{} infers \attacker{} parameters given the observed attack via optimization.
SHERLOC aims to choose a hypothesis $\hat\attacker$ that \textit{behaves consistently with the true $\attacker^*$ based on the observed attacks}.
We evaluate our SHERLOC-inferred attacker $\hat\attacker$ with three distinct measures:
\begin{enumerate}
    \item \textbf{Attack reconstruction:} Does the inferred attacker $\hat\attacker{}$ explain the behavior that \defender{} observed and does it produce attack perturbations similar to the true attacker on unseen points?
    \item \textbf{Attacker recovery:} Among attackers that explain the observations, how much can \defender{} actually learn about the ground-truth attacker $\attacker{}^*$?
    \item \textbf{Downstream defense:} Can a learned attacker hypothesis $\hat\attacker{}$ that is consistent with the observed attacks be useful for downstream defense?
\end{enumerate}

%

\subsection{SHERLOC: Inference Under Non-Identifiability}


The non-identifiability result does not imply that observed attacks contain no information about their source; rather, they are insufficient for choosing a single unique attacker from the equivalence class.
SHERLOC formalizes this distinction by treating attacker recovery as an inference problem within that equivalence class.
Given a set of $k$ observed attacks $\bA_{obs}$ and their corresponding benign inputs $\bX_k$, \defender{} aims to infer the attacker parameters $(\knowledge,\capability,\objective)$ that explain the observed behavior:
\begin{align}
    \hat{\knowledge}, \hat{\capability}, \hat{\objective} &= \arg \min_{\knowledge, \capability, \objective} ~ \sum_{i=1}^{k} \left\Vert \balpha_{obs}^{(i)} -  \balpha_{opt}(\bx_i; \knowledge, \capability, \objective)\right\Vert_2
\end{align}
For each $\balpha_{obs}^{(i)}$, we compute its distance from the optimal attack $\balpha_{opt}(\bx_i; \knowledge, \capability, \objective)$, which depends on both the corresponding benign input $\bx_i$ and the current attacker parameters $(\knowledge, \capability, \objective)$. 
Consequently, each optimization step requires solving $k$ inner problems.

When additional information about the attacker is available, SHERLOC can easily incorporate it to further constrain the set of plausible attacker hypotheses.
For example, prior knowledge about $\knowledge$ can be encoded through regularization toward models considered plausible by \defender{}, while structural knowledge about $\capability$ or $\objective$ can restrict their admissible parameterizations. 
In Section 4, we consider examples of such additional structure, as well as evaluate how they affect attacker recovery.

\section{Experiments}

Our proposed framework is the foundation of a general tool for reverse engineering attacker characteristics. 
Our theoretical analysis establishes that a series of observed attacks cannot, on its own, uniquely identify the attacker $\attacker{}^*$ that generated it.
\ifincludedefense
In what follows, we use SHERLOC to empirically investigate \emph{when} observed attack behavior contains sufficient information to recover an attacker's knowledge, capability, or objective, and when fundamental ambiguities remain.

Our evaluation is organized around three key questions. \textbf{First}, can SHERLOC find attacker parameters that explain the observed attacks, even when those parameters are not uniquely identifiable? \textbf{Second}, when does such reconstruction translate into recovery of the ground-truth attacker parameters \((\knowledge^*, \capability^*, \objective^*)\)? \textbf{Third}, when exact recovery is difficult, can the inferred attacker nevertheless provide useful information for downstream defenses? 
 

To answer these questions, we evaluate SHERLOC across linear and polynomial learners, synthetic and real datasets, and inference of each attacker component $\knowledge$, $\capability$, and $\objective$.
Across these settings, we evaluate two versions of SHERLOC: the generic SHERLOC and SHERLOC+, in which we incorporated additional information as discussed in Section 4.1.
The resulting SHERLOC-inferred attackers are evaluated under 3 lenses: reconstruction, recovery, and defense.
Full details of the experimental setup, including data generation, attack strengths, optimization configurations, and computational resources are provided in Appendix \ref{app:exp_setup}.

\subsection{Incorporating Prior Knowledge}

SHERLOC does not prescribe a particular prior, but can incorporate application-specific information to constrain plausible attacker hypotheses.
To demonstrate this flexibility, we additionally evaluate SHERLOC+, which uses an iid-training assumption for \(K\) and known matrix spectra (but unknown orientations) for \(C\) and \(O\). 
We provide more details on these assumptions in Appendix \ref{app:prior_knowledge} and compare SHERLOC against SHERLOC+ below.

\begin{table*}[b!]
    \centering
    \caption{
        \textbf{Attack reconstruction and parameter recovery on synthetic data.}
        We report  $\texttt{NARS}$, followed by $\texttt{NKRS}$, $\texttt{NCRS}$, and $\texttt{NORS}$ for linear defenders and degree-2 polynomial defenders.
        $\ddagger$ denotes SHERLOC+.
        Note that blank cells reflect the experimental design --- the linear dimensionality sweep was evaluated through $d=16$, while polynomial were evaluated only at selected dimensions.
    }
    \label{tab:synthetic-recovery}
    \small
    \setlength{\tabcolsep}{4.2pt}
    \begin{tabular}{rcccccccc}
    \toprule
    & \multicolumn{2}{c}{$\texttt{NARS}$}
    & \multicolumn{2}{c}{$\capability$ ($\texttt{NCRS}$)}
    & \multicolumn{2}{c}{$\objective$ ($\texttt{NORS}$)}
    & \multicolumn{2}{c}{$\knowledge$ ($\texttt{NKRS}$)} \\
    \cmidrule(lr){2-3}
    \cmidrule(lr){4-5}
    \cmidrule(lr){6-7}
    \cmidrule(lr){8-9}
    $d$ & Linear & Polynomial & Linear & Polynomial & Linear & Polynomial & Linear & Polynomial \\
    \midrule
    2
    & $\geq0.95$ & $\geq0.95$
    & 1.00$\ddagger$ & 1.00
    & 1.00 & 1.00
    & 0.29 & 0.02--0.34 \\
    
    4
    & $\geq0.95$ & $\geq0.95$
    & 0.47$\ddagger$ & 1.00
    & 0.15 & 1.00
    & 0.20 & --- \\
    
    8
    & $\geq0.95$ & ---
    & 0.17$\ddagger$ & ---
    & 0.20 & ---
    & 0.02 & --- \\
    
    12
    & $\geq0.95$ & $\geq0.95$
    & 0.14$\ddagger$ & 1.00
    & 0.25 & 1.00
    & 0.13 & 0.46--0.50 \\
    
    16
    & $\geq0.95$ & $\geq0.95$
    & 0.07$\ddagger$ & 0.99
    & 0.30 & 1.00
    & 0.40 & 0.64--0.72 \\
    \bottomrule
    \end{tabular}
\end{table*}




\subsection{Attack Reconstruction and Attacker Recoverability}

Motivated by the distinction established in Section 2, we separately measure \emph{attack reconstruction} and \emph{attacker recovery}. 
Reconstruction measures whether the inferred attacker reproduces the attacks observed by \defender{}. 
We report the Normalized Attack Reconstruction Score (\texttt{NARS}), where 1 indicates perfect reconstruction and 0 indicates no improvement over the prior attacker. 
We consider an inferred attacker to adequately reconstruct the observations when $\texttt{NARS}\geq0.95$.

We separately measure recovery of the underlying attacker parameters using the Normalized Knowledge, Capability, and Objective Recovery Scores (\texttt{NKRS}, \texttt{NCRS}, and \texttt{NORS}, respectively). 
A recovery score of 1 indicates performance equivalent to knowing the ground-truth parameter, 0 indicates no improvement over the prior, and a negative value indicates an estimate farther from the ground truth than the prior. 
A detailed definition of these evaluation metrics is provided in Appendix \ref{app:eval_metrics}.


\paragraph{Linear learners.} 
Linear models (linear regression, logistic regression) provide the controlled setting predicted by our theoretical analysis. 
For these learners, there is often a fixed optimal attack (\cite{Alfeld_Zhu_Barford_2017}).
Consequently, SHERLOC may be able to determine how an attacker behaves without having sufficient information to determine why it behaves that way.

This phenomenon appears clearly in our results. 
Across linear settings, SHERLOC routinely finds attacker hypotheses that reproduce the observed attacks nearly perfectly ($\texttt{NARS}\approx 1$).
For example, on the $\texttt{uci\_energy}$ and $\texttt{breast\_cancer}$ datasets, SHERLOC achieves perfect $\texttt{NARS} = 1$, yet the corresponding recovery scores of the parameters themselves are strongly negative (see Table \ref{tab:synthetic-recovery}).

Thus, successful reconstruction of the observed attack does not imply identification of the true underlying attacker.
The same limitation worsens quickly with increasing dimension.
On synthetic linear problems, capability recovery with the best-case spectrum information decreases from $1.00$ at $d = 2$, to $0.17$ at $d = 8$, to $0.07$ at $d = 16$. 
Objective recovery likewise falls from $1.00$ at $d = 2$ to $0.1$--$0.30$ for $d = 4$--$16$. 
These results empirically confirm the attacker equivalence class discussed in Section 3, and additionally suggest that the order of that equivalence class increases as dimensionality increases.

\paragraph{Polynomial learners.}
The linear results motivate a natural question: does recovery improve when multiple observed attacks provide additional information about the same attacker?
To answer this question, we must look to non-linear models, out of which we look at the polynomial classifier.
Although these models remain linear in their learned parameters, their gradients with respect to the original input vary across points. 
Multiple attack observations can therefore impose different constraints on the latent attacker.

The contrast is strongest for $\objective$ recovery. 
As shown in Table \ref{tab:synthetic-recovery}, $\texttt{NORS}\approx 1.0$ at every tested synthetic polynomial dimension. 
$\capability$ also becomes highly recoverable: $\texttt{NCRS}\approx 1.0$ through $d = 16$, including $0.99$ on held-out trials at $d = 16$. 

However, SHERLOC uniformly struggles at recovering $\knowledge$, even in the non-linear setting.
As shown in Table \ref{tab:synthetic-recovery}, $\texttt{NKRS}$ is approximately $0.46$--$0.50$ on synthetic data with $d = 12$ dimensions and $0.64$--$0.72$ at $d = 16$.
On $\texttt{pendigits}$, NKRS is $0.38$.
Increasing the number of attacked points from $100$ to $1,000$ to $5,000$ shows that knowledge recovery plateaus, improving only slightly as shown in Table \ref{tab:num-attacks}. 
Thus, observing more attacks can reduce ambiguity for SHERLOC learning $\knowledge$ when applied to non-linear learners, but neither removes it completely nor uniformly across parameters.

\begin{table}[t]
\centering
\caption{
    \textbf{Attacker recovery as the number of observed attacks increases.}
    We report \texttt{NKRS}, \texttt{NCRS}, and \texttt{NORS} achieved by SHERLOC for degree-2 polynomial defenders using increasingly many observed attacks.
    Capability and knowledge results use the $d=12\ $ \texttt{make\_classification} setting; objective results summarize the three evaluated settings. 
}

    \label{tab:num-attacks}
    \small
    \setlength{\tabcolsep}{4pt}
    \begin{tabular}{lccc}
    \toprule
    Target & $n=100$ & $n=1000$ & $n=5000$ \\
    \midrule
    $\capability$ & 0.996 & 0.999 & 0.948 \\
    $\objective$ & 1.000 & 1.000 & 1.000 \\
    $\knowledge$ & 0.46--0.48 & 0.49--0.50 & 0.49--0.50 \\
    \bottomrule
\end{tabular}
\end{table}

\begin{table*}[b]
    \centering
    \caption{
        \textbf{Attack reconstruction and parameter recovery on real data.}
        We report  $\texttt{NARS}$, followed by $\texttt{NKRS}$, $\texttt{NCRS}$, and $\texttt{NORS}$.
        Linear defenders are evaluated on $\texttt{uci\_energy}$ ($d=8$), $\texttt{breast\_cancer}$ ($d=30$), and $\texttt{pendigits}$ ($d=16$); degree-2 polynomial defenders are evaluated on $\texttt{pendigits}$ and $\texttt{breast\_cancer}$. 
        Again, blank cells reflect the experimental design --- a target is omitted when it was not evaluated in that learner-dataset configuration.
    }
    \label{tab:real-recovery}
    \small
    \setlength{\tabcolsep}{3.3pt}
    
    \begin{tabular}{llcc|cc|cc}
    \toprule
    &&
    \multicolumn{2}{c}{$\capability$}
    & \multicolumn{2}{c}{$\objective$}
    & \multicolumn{2}{c}{$\knowledge$} \\
    \cmidrule(lr){3-4}
    \cmidrule(lr){5-6}
    \cmidrule(lr){7-8}
    
    Learner & Dataset
    & NARS & NCRS
    & NARS & NORS
    & NARS & NKRS \\
    \midrule
    
    Linear
    & \texttt{uci\_energy}
    & 1.000 & $-5.0$--$-6.4$
    & --- & 0.60
    & --- & $-0.22$--$-0.70$ \\
    
    Linear
    & \texttt{breast\_cancer}
    & 1.000 & $-11.7$
    & --- & ---
    & --- & 0.13 \\
    
    Linear
    & \texttt{pendigits}
    & --- & $-1.9$--$-8.9$
    & --- & $0.00$--$-16.7$
    & --- & 0.07--0.13 \\
    
    \midrule
    
    Poly-2
    & \texttt{pendigits}
    & --- & 0.19--0.49
    & --- & 0.87--0.95
    & --- & 0.31 \\
    
    Poly-2
    & \texttt{breast\_cancer}
    & 0.98 & $-4.4$
    & --- & --- 
    & --- & --- \\
    
    \bottomrule
\end{tabular}
\end{table*}

\subsection{Downstream Defense Applicability}

The preceding experiments treat recovery partially as a forensic objective: how accurately can \defender{} infer the latent parameters of \attacker{}? 
However, exact identification need not be necessary for every downstream task. 
In this section, we aim to answer whether an inferred attacker hypothesis that captures the observed attack behavior can still be useful for downstream defense, despite the latent attacker parameters being only partially recovered.

We compare adversarial regularization using (i) a baseline attacker model $\attacker{}_0$ representing the information available to \defender{} without observing the attacks, (ii) the attacker $\hat\attacker{}$ inferred by SHERLOC, and (iii) the ground-truth attacker $\attacker{}^*$.
The ground-truth condition represents the benefit available under perfect attacker information. 
Because this comparison is only meaningful when perfect attacker information improves the defense over the prior, we restrict interpretation to settings with positive \textit{defense headroom}, defined as the difference between the ground-truth and baseline defense effectiveness.

We measure downstream utility of SHERLOC with \textit{defense gain}:
\begin{align}
    \texttt{GAIN} = \frac{\texttt{DEF}(\hat\attacker{})-\texttt{DEF}(\attacker{}_0)}{\texttt{DEF}(\attacker{}^*) - \texttt{DEF}(\attacker{}_0)}
\end{align}
When $\texttt{GAIN}=1$ indicates that the inferred attacker $\hat\attacker{}$ fully achieves the benefit of defending with the knowledge of the ground-truth attacker $\attacker{}^*$, and $\texttt{GAIN}=0$ indicates no improvement over the baseline.

Where measurable headroom exists, adversarial regularization with the inferred attacker $\hat\attacker{}$ captures most of the benefit obtained using the ground-truth attacker. 
Across $\knowledge$ recovery trials with a linear \defender{}, $\texttt{GAIN}$ is $0.99-1.00$ on $\texttt{uci\_energy}$, $0.95–0.97$ on $\texttt{pendigits}$, and $0.89-0.99$ on $\texttt{breast\_cancer}$ (see Table \ref{tab:defense}). 
For the polynomial \defender{}, $\texttt{GAIN}$ is $1.00$ on $\texttt{pendigits}$. 
The linear pendigits result is especially illustrative.
Although $\texttt{NKRS}$ is approximately zero, the inferred attacker captures $95-97\%$ of the available defensive benefit.
Thus, poor recovery does not imply poor downstream utility.

For objective inference on $\texttt{pendigits}$, defense gain ranges from $0.72$ to $1.05$, while a control that receives the same structural information but no observed attacks achieves only $0.15-0.20$. 
This comparison indicates that the observed attacks contribute information beyond the supplied structure. 
Capability-based defense cannot be interpreted reliably in our current experiments because the available headroom for defense is generally below $0.02$, meaning there is not adequate room for improving robustness.

\begin{table}[t]
    \centering
    \caption{
        \textbf{Downstream defense using SHERLOC-inferred attackers.}
        We report $\texttt{GAIN}$ for adversarial regularization on real datasets in settings with measurable defensive headroom.
        Only settings with interpretable headroom are reported.
    }
    \label{tab:defense}
    \small
    \setlength{\tabcolsep}{4pt}
    \begin{tabular}{lcc}
    \toprule
    Setting & Headroom & Defense gain \\
    \midrule
    $\knowledge$, linear \texttt{uci\_energy} & $+0.14\ R^2$ & 0.99--1.00 \\
    $\knowledge$, linear \texttt{pendigits} & $+0.065\ p_{\rm true}$ & 0.95--0.97 \\
    $\knowledge$, linear \texttt{breast\_cancer} & $+0.04\ p_{\rm true}$ & 0.89--0.99 \\
    $\knowledge$, poly-2 \texttt{pendigits} & $+0.03\ p_{\rm true}$ & 1.00 \\
    $\objective$, poly-2 \texttt{pendigits} & $+0.02$--$0.03\ p_{\rm true}$ & 0.72--1.05 \\
    \bottomrule
    \end{tabular}
\end{table}

Overall, these results expose a useful separation between forensic identification and defensive utility. 
Recovering the exact latent attacker parameters requires distinguishing among behaviorally equivalent hypothesis attackers, whereas the defense may only require the attack-relevant behavior that those hypotheses share with the ground-truth. 
Therefore, SHERLOC can provide useful defensive information even in settings where the attacker non-identifiability discussed in Section 3 prevents exact parameter recovery.

\section{Related Work}


Adversarial learning is a well-established research area.
We direct the reader to \citet{Biggio_2018}, \citet{Vorobeychik2018AdversarialML}, and other surveys \citep{chakraborty2018adversarialattacksdefencessurvey, 10.1145/3485133} for an overview of the field.
In this paper, we focus on the deployment-time attacker, which is an especially well-studied class of adversarial attacks.
The deployment-time attacker is particularly relevant in the field of computer vision, where imperceptible input perturbations have been shown to reliably induce misclassification in high-capacity models \citep{szegedy2014intriguingpropertiesneuralnetworks, moosavidezfooli2018robustnesscurvatureregularizationvice}.
In contrast to much of this literature, our contribution is not the proposal of a new attack or defense mechanism; rather, we develop a framework that enables the defender to infer key properties of the attacker from observed attacks.


Reverse Engineering of Deception (RED) is an emerging area of adversarial learning that aims to extract insights into an adversary by examining observed attacks.
Prior work in this area has shown that adversarial perturbations can exhibit structured, attacker-dependent patterns (characteristic geometric or statistical properties of the perturbation) that can be exploited to identify or classify the underlying attack mechanism \citep{10.1007/978-981-99-8082-6_7, 10446989}.
These approaches demonstrate that observed attacks encode informative signals about the attacker.
However, they are typically tailored to specific attack families ($\ell_p$-constrained) and rely on strong assumptions about the attacker’s optimization procedure \citep{pmlr-v162-thaker22a}.

Many RED approaches are additionally limited to particular input domains, such as images.
For instance, prior work has shown that it is possible to reverse engineer the exact perturbations made in order to produce an adversarial image \citep{gong2022reverseengineeringimperceptibleadversarial}.
It is also possible to attribute an adversarially attacked image to a particular
attack toolchain or family \citep{doi:10.2352/ISSN.2470-1173.2021.4.MWSF-300}.
In contrast, our framework is domain-agnostic and does not assume, but \textit{learns} the attacker's perturbation structure.
This abstraction allows our approach to generalize across attack types and learning models.

\section{Conclusions}

We introduced SHERLOC, a novel framework for reasoning backward from observed attacks to an attacker that could have produced them.
SHERLOC characterizes an attacker through their knowledge (\(\knowledge\)), capability (\(\capability\)), and objective (\(\objective\)), then answers the question of what an adversary reveals through their actions.

SHERLOC does not recover the exact ground-truth attacker --- however, it is able to successfully and reliably reconstruct a hypothesis attacker that behaves consistently to the observed attacks.
For linear learners, SHERLOC can reconstruct attacks nearly perfectly while recovering very little about the underlying attacker. 
When attacks provide richer, input-dependent behavioral information, as with polynomial learners, capability and objective become substantially more recoverable, while knowledge remains comparatively difficult to identify. 
These results empirically demonstrate the effect of attacker non-identifiability --- a non-trivial equivalence class of distinct attackers that produce the same attacks prevent exact recovery.

However, exact identification is not necessary for downstream utility.
Even when SHERLOC recovers an attacker that differs substantially from the ground-truth, the SHERLOC-inferred attacker proves useful in downstream defense.
In settings with measurable defensive headroom, adversarial regularization constructed using SHERLOC-inferred attackers captured most of the robustness improvement attainable using the ground-truth attacker.
Thus, a defender may not need to know exactly \emph{who} the attacker is to learn enough about \emph{how} they attack to respond effectively.

SHERLOC provides both a practical tool and a new outlook on Reverse Engineering of Deception (RED). 
It reframes adversarial learning from only asking how to withstand an attack to also asking what can be learned from the attack that occurred. 
The resulting fundamental question of \textit{what information is leaked by an attacker when they manipulate data} opens a broader direction for understanding and responding to dynamic adversaries.



\subsection{Future Work}


SHERLOC opens a broad space of questions for the field of Reverse Engineering of Deception (RED). 
Rather than prescribing a single attacker model, prior, or inference setting, the framework provides a common language for asking what an adversary reveals through their behavior, what remains fundamentally ambiguous, and what additional information can resolve that ambiguity. 
The settings studied here therefore represent only a small portion of the possible design space.

One particularly important direction is the development of effective \emph{priors} for attacker inference. 
Our experiments demonstrate that additional defender knowledge can substantially alter recoverability, but the appropriate prior will necessarily depend on the application and domain. 
Future work can investigate other forms of prior knowledge.
More broadly, understanding the interaction between information revealed by an attack and information supplied by the defender is itself a fundamental problem for RED.
The framework can also be extended well beyond the deployment-time attacks considered in this paper. 
In particular, training-time attacks provide a natural next setting.

These directions are illustrative rather than exhaustive. 
Different learners, attacker parameterizations, and downstream uses each induce new reverse-engineering questions. 
We view SHERLOC as a starting point for systematically asking these questions, and hope that it enables a broader community effort around RED towards not only developing new defenses, but understanding what adversaries reveal about themselves by acting in the first place.

\bibliography{refs}
\bibliographystyle{iclr2027_conference}

\appendix
\section{Parameterization of \attacker{} for Linear Models}

\label{app:params}

Suppose \defender{} is a generalized least-squares linear model (i.e., linear regression, logistic regression) with the prediction function $f: \mathbb{R}^d \rightarrow \mathbb{R}^q$:
\begin{align}
f(\bx) &= \model \bx
\end{align}
\noindent where $\model \in \mathbb{R}^{q \times d}$ is the weight matrix of the model.

Suppose \attacker{} performs an \textit{evasion attack under Mahalanobis constraints}.
\attacker{} aims to solve the following optimization problem: 
\begin{align}
    \balpha_{opt}(\bx ; \attacker{}) &= \arg \max_{\balpha} \Vert \model' (\bx + \balpha) - \model' \bx \Vert_{\bW}^2 \\
    &= \arg\max_{\balpha} \Vert \model' \balpha \Vert_{\bW}^2 \\
  &\text{s.t.~} ||\balpha||_{\bC} \leq c
\end{align}

where
\begin{itemize}

\item \(\model' \in \mathbb{R}^{q\times d}\) is \attacker{}'s estimate of \defender{}'s model $\model$.

\item \(\bW \in \mathbb{R}^{q\times q}\) is a positive definite matrix defining the Mahalanobis loss of \attacker.

\item \(\bC \in \mathbb{R}^{d\times d}\) is a positive definite matrix defining (with positive scalar \(c\)) the Mahalanobis constraint of \attacker.

\end{itemize}

In this case, we parameterize \attacker{} as follows:
\begin{align}
    \label{param:1}
    \knowledge &= \{\model'\},  \capability = \{\bC, c\}, \objective = \{\bW\}
\end{align}

\noindent Let $d$ denote the input dimension and $q$ denote the output dimension of the model. Hence, we have the following \attacker{} parameters:
\begin{align}
    \model &\in \mathbb{R}^{q \times d} \\
    \bC &\in \mathbb{R}^{d \times d} \\
    c &\in \mathbb{R} \\
    \bW &\in \mathbb{R}^{q \times q}
\end{align}

\section{Proof of Theorem 3.2}
\label{app:thm3.2_proof}

\textbf{Theorem 3.2. }  
\textbf{General non-identifiablity of attacker.}

In the case of a linear \defender{}, the following hold for any attack $\balpha$:
    \begin{itemize}
        \item $\forall \knowledge, \capability,~ \exists \objective ~ \text{s.t.} ~ \attacker \in \mathcal{S}_{\balpha}$
        \item $\forall \knowledge, \objective,~ \exists \capability ~ \text{s.t.} ~ \attacker \in \mathcal{S}_{\balpha}$
        \item $\forall \capability, \objective,~ \exists \knowledge ~ \text{s.t.} ~ \attacker \in \mathcal{S}_{\balpha}$
    \end{itemize}

\noindent \textbf{Proof.}

We will prove each of the three statements below.

\noindent \textbf{1.} $\forall \knowledge, \capability,~ \exists \objective ~ \text{s.t.} ~ \attacker \in \mathcal{S}_{\balpha}$

This statement can be proved for any invertible model matrix $\model$. 

We can construct $\bV$ such that the largest singular value of $\bV \model \bG^{-1}$ corresponds to the vector $\frac{1}{c} \bG \balpha$. 

First, construct a placeholder matrix $\bV$ (utilizing singular value decomposition):
\begin{align}
    \bV &= \hat{\mathbf{U}}\hat{\mathbf{\Sigma}}\hat{\bV}^{\top} \\
    &= 
    \begin{bmatrix}
    & & \\
    & & \\
    \frac{1}{c} \bG \balpha^* & ... & \\
    & & \\
    & &
    \end{bmatrix}
    \hat{\mathbf{\Sigma}}\hat{\bV}^{\top} \\
\end{align}

Then, adjust $\bV$ to account for $\model$ and $\bG^{-1}$.
\begin{equation}
    \bV' = \bV \bG \model^{-1}
\end{equation}

Then, we have:
\begin{align}
    \lVert \bV' \model \bG^{-1} \rVert &= \lVert \bV \bG \model^{-1} \model \bG^{-1} \rVert \\
    &= \lVert \bV \bG \bG^{-1} \rVert \\
    &= \lVert \bV \rVert
\end{align}

Thus, the largest singular vector $s_{1}$ of $\bV' \model \bG^{-1}$ is the largest singular vector of $\bV$, which is clearly $\frac{1}{c} \bG \balpha$. 
Hence, the optimal attack is  $\balpha$.

Let $\bW = \bV'^{\top} \bV'$. 
Then, we have picked the objective $\objective = \{ \bW \}$ such that the attacker constraints and model belief result in the attack $\balpha$. 
Hence, $\forall \knowledge, \capability,~ \exists \objective ~ \text{s.t.} ~ \attacker \in \mathcal{S}_{\balpha}$. $\square$

\hfill

\noindent \textbf{2.} $\forall \knowledge, \objective,~ \exists \capability ~ \text{s.t.} ~ \attacker \in \mathcal{S}_{\balpha}$

Let $\model$ be the model that the attacker assumed of the defender. 

Let $\bW$ = some positive definite matrix. This defines the goal of the attacker, i.e., $\objective$.

We know that $||\balpha||_{\bC} \leq c$. 
That is, the attack must be within the capabilities of the attacker. 

Consider the goal of the attacker --- to maximize the defender loss on the adversarial example. 
Assume that the attacker's attack was approximately optimal. 
Then, 

\begin{equation}
\begin{split}
& {\lVert \model \balpha \rVert}^2_{\bW} \\
& = \balpha^{\top} \model^{\top} \bW \model \balpha \quad \text{was maximized.}
\end{split}
\end{equation}

Let $\bC = \model^{\top} \bW \model$. 
Then, 

\begin{align}
{\lVert \balpha \rVert}_{\bC} &= \balpha^{\top} \bC \balpha \\
&= \balpha^{\top} \model^{\top} \bW \model \balpha \\
&= \left( \model \balpha \right)^{\top} \bW \left( \model \balpha \right) \\
&= {\lVert \model \balpha \rVert}^2_{\bW}
\end{align}

Let $c = {\lVert \model \balpha \rVert}^2_{\bW}$. 

Then, the attacker is constrained by the exact value it is maximizing. 
The upper bound is also the maximum value of the objective.
Hence, we have found a $\capability$ such that for any $\knowledge,~\objective$, the attacker would learn exactly $\balpha$, i.e., $\attacker \in \mathcal{S}_{\balpha}$. $\square$

\hfill

\noindent \textbf{3. } $\forall \capability, \objective,~ \exists \knowledge ~ \text{s.t.} ~ \attacker \in \mathcal{S}_{\balpha}$

The proof follows similarly to the proof of statement 1.
We can construct $\model$ such that the largest singular value of $\bV \model \bG^{-1}$ corresponds to the vector $\frac{1}{c} \bG \balpha^*$. 

First, construct a placeholder model $\model$ (utilizing singular value decomposition):
\begin{align}
    \model &= \hat{\mathbf{U}}\hat{\mathbf{\Sigma}}\hat{\bV}^{\top} \\
    &= 
    \begin{bmatrix}
    & & \\
    & & \\
    \frac{1}{c} \bG \balpha^* & ... & \\
    & & \\
    & &
    \end{bmatrix}
    \hat{\mathbf{\Sigma}}\hat{\bV}^{\top} \\
\end{align}

Then, adjust $\model$ to account for $\bV$ and $\bG^{-1}$.
\begin{equation}
    \model' = \bV^{-1} \model \bG
\end{equation}

Then, we have:
\begin{align}
    \lVert \bV \model' \bG^{-1} \rVert &= \lVert \bV \bV^{-1} \model \bG \bG^{-1} \rVert \\
    &= \lVert \model \rVert
\end{align}

Thus, the largest singular vector $s_{1}$ of $\bV \model \bG^{-1}$ is the largest singular vector of $\model$, which is clearly $\frac{1}{c} \bG \balpha$.
Hence, the optimal attack is $\balpha$.

Then, we have picked the attacker's assumed defender model $\knowledge = \{\model\}$ such that the attacker constraints and objective result in the attack $\balpha$. 
Hence, $\forall \capability, \objective,~ \exists \knowledge ~ \text{s.t.} ~ \attacker \in \mathcal{S}_{\balpha}$. $\square$

\section{Experimental Setup}
\label{app:exp_setup}

SHERLOC is evaluated across both linear and polynomial learners. 
Linear experiments use logistic and ridge regression. 
Polynomial experiments use logistic regression over degree-2 or degree-3 polynomial features. 
Both model families are implemented via PyTorch. 

This evaluation is performed across several datasets.
First, synthetic experiments are used to systematically vary input dimensionality and the amount of observed attack behavior. 
Linear experiments cover dimensions $d\in{2,4,8,12,16}$ for the reported results in Section 4. 
Polynomial experiments extend selected configurations through $d=32$, with reported degree-2 experiments at $d\in{2,4,12,16,24,32}$ depending on the inference target. 

To test whether recovery is limited by the quantity of observed behavior, we additionally varied the number of attacked points for degree-2 polynomial learners. 
We evaluate $k\in{100,1000,5000}$ observed attacks. 
For $\knowledge$ and $\capability$, the reported comparison uses the $d=12$ $~ \texttt{make\_classification}$ setting; the $\objective$ comparison aggregates the three evaluated settings.

In addition to this synthetic setting, we evaluate SHERLOC on three real datasets: \texttt{uci\_energy} ($d=8$, regression), \texttt{breast\_cancer} ($d=30$), and \texttt{pendigits} ($d=16$). 
Linear learners are evaluated on all three datasets; degree-2 polynomial learners are evaluated on \texttt{pendigits} and \texttt{breast\_cancer}. 
Real-data results are reported as medians over trials.

Detective hyperparameters (optimizer settings/configuration) are selected using synthetic experiments and subsequently applied to the real datasets without dataset-specific tuning. 
For each dataset, we evaluate attack strengths $\epsilon$ for which the attack measurably affects the learner. 
Settings in which the attack is too weak i.e., does not change model performance are excluded from the reported results.

Our experiments include both generic SHERLOC and informed SHERLOC+ detectives. 
The generic configuration uses six optimization restarts.
For $\knowledge$, we additionally evaluate an iid-training assumption formalized in Appendix \ref{app:prior_knowledge}: the attacker's proxy model is trained on separate data drawn from the same underlying distribution as the defender's data. 

For $\capability$ and $\objective$, we additionally evaluate spectrum-informed detectives. 
These detectives are supplied with the eigenvalues of the corresponding ground-truth matrix but not its eigenvectors. 
We also use a no-data control for the spectrum-informed setting. 
This detective receives the spectrum but no observed attacks, allowing us to distinguish information supplied directly by the spectrum from information learned from attack behavior.

To test whether incomplete recovery is caused simply by insufficient optimization, we perform a heavier-optimization diagnostic for $\knowledge$, increasing the optimization budget to three restarts of 1000 iterations. 
We compare this configuration against the standard detective rather than treating additional computation as a distinct inference assumption. 
Together with the experiment varying the number of observed attacks, this diagnostic separates limitations attributable to search effort or sample quantity from ambiguity in the information exposed by the attacks.

Unless otherwise noted, reported real-data values are medians over all trials. 
Synthetic summary tables report the \textbf{best} detective satisfying the reconstruction criterion for each setting, summarized by the median across the corresponding settings. 
Linear synthetic values are held-out. 
Polynomial high-dimensional cells are based on fewer trials, with held-out confirmation for the reported (d=16) and (d=24) synthetic experiments. 
These differences in replication are noted when interpreting high-dimensional polynomial results.

\section{Incorporating Prior Knowledge}
\label{app:prior_knowledge}

To demonstrate SHERLOC's ability to incorporate additional information in inferring $\hat\attacker$, we evaluate SHERLOC+ in addition to generic SHERLOC.

To learn $\knowledge$, SHERLOC+ encodes the assumption that \attacker{} constructs its proxy model using data drawn from the same underlying distribution as \defender{}. 
\defender{} knows only its own model $\model$ and the dataset $\bD = (\bX, \bY)$ it was trained on.
SHERLOC+ assumes \attacker{} performs a traditional transfer attack using a surrogate model $\model'$ trained on some dataset $\bD' = (\bX', \bY')$ drawn i.i.d. from the same underlying data distribution $\datap$.

Rather than assuming that \defender{} knows \attacker{}'s model parameters, this induces a prior over plausible knowledge models: models that could reasonably have resulted from training on another sample from the same distribution are preferred over arbitrary models. 
Since $\model'$ depends only on the training dataset $\bD'$, we have
\begin{align}
P(\model' \mid \datap)
= \int P(\model' \mid \bD')P(\bD' \mid \datap) \, d\bD'.
\label{eq}
\end{align}

Suppose both \defender{} and \attacker{} use a generalized least-squares linear regressor. 
We assume the data-generating process is
\begin{align}
    \bY= \bX \model^{*\top}+\varepsilon , \quad \varepsilon\sim\mathcal N(0,\Sigma_\varepsilon)
\end{align}
where $\model^* \in\mathbb{R}^{q\times d}$ is the unknown true weight matrix.

The generalized least-squares estimator is linear in the observed responses and can therefore be written as:
\begin{align}
    \model' &= \bY'^{\top} A(\bX'), \\
    A(\bX') &= \Omega^{-1}\bX'(\bX'^{\top}\Omega^{-1}\bX')^{-1}
\end{align}
and $\Omega \in \mathbb{R}^{\times}$ is the covariance matrix of the observation errors.
Substituting the data-generating model, 
\begin{align}
    \model' &= (\bX'\model^{*\top} +\varepsilon)^{\top} A(\bX') \\
    &= \model^* (\bX'^{\top} A(\bX')) +\varepsilon^{\top} A(\bX')
\end{align}

Since $A(\bX')$ is deterministic and $\varepsilon$ is Gaussian, $\model'$ is an affine transformation of a Gaussian random matrix and is therefore itself Gaussian. 
Therefore, we have:
\begin{align}
    P(\model' | \datap) = \mathcal{N}(\pmb{\mu}, \pmb{\Sigma})
\end{align}
where the mean and covariance are determined by the true $\model^*, \Sigma_{\varepsilon},$ and the estimator$A(\bX')$.

In practice, $\model^*, \Sigma_{\varepsilon}$, and the underlying data distribution $\datap$ are unknown.
Hence, \defender{} approximates them using its own training data $\bD = (\bX, \bY)$.
Specifically, we substitute \defender{}'s learned model $\model$ for $\model^*$, estimate the noise covariance $\hat\Sigma_{\varepsilon}$ from the residuals on $\bD$, and compute the corresponding estimator matrix $A(\bX)$.
This yields the following approximate prior:
\begin{align}
    P(\model') = \mathcal{N}(\pmb{\mu}, \pmb{\Sigma})
\end{align}
We use this distribution as the prior over \attacker{}'s knowledge parameter $\knowledge$ to enable inference for non-linear models.

For nonlinear models such as polynomial regressors and classifiers, an analogous prior is generally intractable.
Instead, we define an implicit prior. 
We introduce a \textit{regularizer} that addresses the attacker non-identifiability problem by favoring surrogate models whose predictions are more consistent with \defender{}'s model.

We estimate the disagreement between \defender{}'s model and an estimated attacker surrogate empirically via the mean squared output difference over $n$ samples drawn i.i.d. from the true data distribution $\datap$:
\begin{align}
    \textbf{DIS}_{\defender{}}(\knowledge) = \frac{1}{n} \sum_{i=1}^{n} \vert f(x_i) - f_{\knowledge'}(x_i) \vert^2
\end{align}
\noindent
 where $f(\cdot)$ and $f_{\knowledge}(\cdot)$ denote the outputs of \defender{}'s model and the attacker surrogate model, respectively.

The defender now solves the following optimization problem.
\begin{definition}
    \label{def:mlp_k_obj}
    Given the set of $k$ observed attacks $\bA_{obs}$, the MLP \defender{} that knows \attacker{}'s true $\capability^*, \objective^*$ and aims to infer \attacker{}'s $\knowledge^*$ has the following objective function:
    \begin{align}
        \hat{\knowledge}
        &= \arg \max_{\knowledge} ~
        \sum_{i=1}^{k} \left\Vert\balpha_{obs}^{(i)} - \balpha_{opt}^{(i)}\right\Vert_2 - \lambda \cdot \textbf{DIS}_{\defender{}}(\knowledge)\\
        &\text{s.t.}~ \balpha_{opt}^{(i)} = \arg\max_{\balpha \in \mathcal{A}_{\capability^*}} g_{ \objective^*} (f_{\knowledge}, \bx_i, \balpha), \quad i = 1, \dots, k
    \end{align}
\end{definition}

For \(\capability\), SHERLOC+ utilizes a \emph{spectrum-informed} setting in which \defender{} knows the eigenvalues of the attacker's capability matrix, but not its eigenvectors. 
Equivalently, \defender{} knows the \emph{shape and scale} of the attacker's capability ellipsoid, but not its orientation in feature space. 
Thus, \defender{} may know how strongly the attacker's ability to manipulate inputs varies across directions without knowing \emph{which} feature directions correspond to those capabilities. 
SHERLOC+ must infer this unknown orientation from the observed attacks.

For $\objective$, SHERLOC+ analogously considers a spectrum-informed setting. 
That is, \defender{} is given the objective spectrum, but not its orientation. 
Intuitively, \defender{} knows the relative scale with which the attacker values different directions in the output space, but does not know which output directions receive those weights. 
SHERLOC+ uses the observed attacks to infer this missing orientation.

Together, these assumptions demonstrate information that can be supplied to SHERLOC to improve inference. 
The experiments in Section 4 separately evaluate SHERLOC and SHERLOC+, allowing us to measure how much attacker information is revealed by the attacks themselves and how much reconstruction and recovery respectively depend on additional defender knowledge.

\section{Evaluation Metrics}
\label{app:eval_metrics}

We first evaluate whether the inferred attacker reproduces the observed attack behavior.
The evaluation is done on both the seen data (attacked points) and unseen data (held-out set).
This allows a separation of attackers that can reconstruct the exact observed attacks vs. the same behavior as the ground-truth attacker on unseen data.

First, we look at objective reduction, computed entirely on seen data (the attack set the detective fit to):

  $$\text{reduction} = 1 - \frac{L(\hat\theta;,X_\text{attack})}{L(\theta_0;,X_\text{attack})}$$

This is the literal quantity the outer loop minimizes, evaluated at the final SHERLOC-inferred attacker $\hat\attacker$.    

Then, we compute Normalized Attack Reconstruction Score ($\texttt{NARS}$) and the component-wise scores ($\texttt{NKRS}, \texttt{NCRS}, \texttt{NORS}$) on held-out data.

The reconstruction score \texttt{NARS}, is normalized relative to a naive baseline: $1$ represents perfect reconstruction of the observed attacks and $0$ represents no improvement over the baseline. 
We use $\texttt{NARS}\geq0.95$ as the criterion for successful reconstruction. 
When comparing detective configurations for recovery, reported synthetic summaries select among configurations satisfying this reconstruction criterion.

\end{document}